\documentclass[12pt]{article}
\usepackage{amsmath}
\usepackage{graphicx,psfrag,epsf}
\usepackage{enumerate}
\usepackage{natbib}
\usepackage{url} 
\usepackage{float,cuted,threeparttable}
\usepackage{array}
\usepackage[skip=2pt, labelfont=bf]{caption}
\usepackage{enumitem}
\usepackage{multirow,bm}
\usepackage{tikz}
\usepackage{pgfplots}
\usepackage{amsthm}
\usepackage{amssymb,mathtools}

\usepackage{algpseudocode}
\usepackage{algorithm,setspace}
\algnewcommand{\Initialize}[1]{%
	\State \textbf{Initialize:}
	\hspace*{\algorithmicindent}\parbox[t]{.8\linewidth}{\raggedright #1}
}
\algdef{SE}[SUBALG]{Indent}{EndIndent}{}{\algorithmicend\ }%
\algtext*{Indent}
\algtext*{EndIndent}

\usetikzlibrary{positioning,calc,arrows,plotmarks}
\usepackage[font=small,skip=0pt]{subcaption}
\newtheorem{theorem}{Theorem}

\newcommand{\blind}{0}

\begin{document}

\def\spacingset#1{\renewcommand{\baselinestretch}%
{#1}\small\normalsize} \spacingset{1}


\if0\blind
{
  \title{\bf Lagged Exact Bayesian Online Changepoint Detection with Parameter Estimation}
  \author{Michael Byrd, Linh Nghiem, Jing Cao \hspace{1cm}\\
    Department of Statistical Science \\ Southern Methodist University
    \date{} }
  \maketitle
} \fi

\if1\blind
{
  \bigskip
  \bigskip
  \bigskip
  \begin{center}
    {\Large \bf Lagged Exact Bayesian Online Changepoint Detection with Parameter Estimation}
\end{center}
  \medskip
} \fi

\bigskip
\begin{abstract}
Identifying changes in the generative process of sequential data, known as changepoint detection, has become an increasingly important topic for a wide variety of fields. A recently developed approach, which we call EXact Online Bayesian Changepoint Detection (EXO), has shown reasonable results with efficient computation for real time updates. The method is based on a \textit{forward} recursive message-passing algorithm. However, the detected changepoints from these methods are unstable.
We propose a new algorithm called Lagged EXact Online Bayesian Changepoint Detection (LEXO) that improves the accuracy and stability of the detection by incorporating $\ell$-time lags to the inference. The new algorithm adds a recursive \textit{backward} step to the forward EXO and has computational complexity linear in the number of added lags. Estimation of parameters associated with regimes is also developed. Simulation studies with three common changepoint models show that the detected changepoints from LEXO are much more stable and parameter estimates from LEXO have considerably lower MSE than EXO. We illustrate applicability of the methods with two real world data examples comparing the EXO and LEXO.

\end{abstract}

\noindent%
{\it Keywords:}  recursion, regime switching, forward-backward
\vfill

\newpage
\spacingset{1.45} 
\section{Introduction}
\label{sec:intro}

Changepoint detection is the process of detecting abrupt changes in the generative process of sequential data.  Analysis of this form has been used in many applications including climate change \cite{reeves2007review}, \cite{li2012multiple}, robotics \cite{niekum2015online}, and telecommunication \cite{shields2017application}. We consider the case where the data sequence is assumed to be grouped into non-overlapping blocks, as in \cite{barry1992product}, each of which is often called a regime.  Each regime is assumed to be generated from a fixed process. Boundaries between two adjacent regimes are called changepoints;  the goal of changepoint detection is to determine where these boundaries are located in observations of sequential data. While detecting changepoints, it is also often of interest to estimate the parameters associated with regimes in order to quantify the changes that have occurred.

Changepoint detection methods can be divided into offline and online methods. An offline method, like \cite{killick2012optimal}, \cite{saatcci2010gaussian} and \cite{maidstone2017optimal} needs to observe the whole sequence of data before aiming to make an inference about the generative process of the sequence.  On the contrary, an online method aims to update the inference at each time point with the newly observed data value from the sequence. In particular, from a Bayesian viewpoint, online changepoint detection started with the two similar methods developed by \cite{adams2007bayesian} and \cite{fearnhead2006exact}. The two methods can be inferred directly from the other; more importantly, both methods give exact and online Bayesian posterior inference, hence forth referred to as the EXact Online Bayesian Changepoint Algorithm (EXO). In the past decade, EXO has spawned a wealth of literature developing and expanding on the methodology; examples being particle filters \cite{fearnhead2007line}, regime classification \cite{ranganathan2012pliss},  online Thompson sampling \cite{mellor2013thompson}, and hyperparameter tuning \cite{caron2012line}. EXO is shown to be effective in detecting changepoints in real world examples like in \cite{adams2007bayesian} and \cite{fearnhead2006exact}.  

We propose the Lagged EXact Online Bayesian Changepoint Algorithm (LEXO) that improves the EXO by introducing a time lag into the online inference. Whereas EXO uses the data from time 1 to $t$ to make an inference on the current run length only at time $t$, our proposed method incorporates a backward inference on the run length at time $t-\ell$, for any $\ell \in \mathbb{N}$. This backward inference has been used extensively in dynamic linear time series models, like Kalman smoothers, and Hidden Markov Model smoothers, see \cite{douc2014nonlinear}. As we will illustrate, the main advantage of incorporating this backward inference is to stabilize the inference made by EXO. This stability is important for estimating parameters of regimes; indeed, parameter estimation  from LEXO has a considerably smaller MSE than that from EXO. Moreover, we show that LEXO can be computed recursively with  computation complexity being linear in the number of lags.

The paper is developed in the following way. In Section 2, we review the theory of EXO and develop the theory for LEXO for the run length. In Section 3, we develop parameter estimation for both EXO and LEXO. 
In Section 4, we illustrate the performance for the EXO and LEXO via an extensive simulation study for three common changepoint models. We illustrate the applicability of LEXO in Section 5 with two real data examples and conclude the paper in Section 6.

\section{Run-length estimation}
\subsection{EXO}

Let stochastic process $\{X_{\tau(t)}\}_{t = 1}^\infty$ be observed sequentially as $x_t$ in equally spaced intervals of time, $t = 1,2,\ldots$, of an unknown length.  The realizations of the process form non-overlapping regimes, where each observation is generated strictly from one regime; denote $\tau(t)$ to index each regime, where time $t$ uniquely identifies the regime it was generated from. Further, assume that observations generated within the same regime are $\textit{iid}$, and also independent from observations in other regimes; this formulation is known as a product partition model, see \cite{barry1992product}.

For a given regime starting at $t^*$, we assume that 
\[
x_{t^*}, x_{t^* + 1}, \ldots \overset{i.i.d.}{\sim} f(x \vert \boldsymbol{\eta}_{\tau(t^*)}).
\]
Here, $f$ denotes the probability distribution with regime specific parameter $\boldsymbol{\eta}_{\tau(t^*)}$.  We additionally assume that the underlying distribution does not change with time, only the underlying parameter of the distribution.  A changepoint is defined as the starting point for a new regime.  The run length at a given time, $r_t$, is defined as the number of steps taken since the last changepoint occurred.  Formally, a changepoint occurs at $t$ when $\tau_{t} = \tau_{t-1} + 1$ and the run length at $t$ is $r_t = t - t^*$ for $t \geq t^*$, where $t^*$ is the latest changepoint before time $t$.

The EXO algorithm computes the exact posterior distribution for the current run length at each time point, and changepoints are learned from these distributions. This gives a flexible model for a practitioner to make informed decisions about the occurrence of a changepoint at each time point.  For instance, if the concentration of the mass of the run length's distribution shifts greatly toward 0 at time $t+1$ after being mostly concentrated around the max possible run length at time $t$, then this would indicate a change in the regime.

Following the same notation as in \cite{adams2007bayesian}, we denote $r_t$, $\boldsymbol{x}_{a:b}$, and $\boldsymbol{x}_t^{(r)}$ as the current run length at time $t$, the set of observations from time $a$ to time $b$, and the set of observations associated with the regime at time $t$, respectively. Then, the goal is to find
\begin{equation}
P(r_t |\boldsymbol{x}_{1:t}) = \dfrac{P(r_t, \boldsymbol{x}_{1:t})}{P(\boldsymbol{x}_{1:t})},
\label{eq:changepoint}
\end{equation}
corresponding to the posterior distribution of the current run length at time $t$ given all the observed data up to time $t$. As shown in \cite{adams2007bayesian} and \cite{fearnhead2006exact}, exact posterior inference can be found by noting
\begin{equation}
\begin{split}
\theta_{t} & \triangleq P(r_t,\boldsymbol{x}_{1:t}) =  \sum_{r_{t - 1}=0}^{t-1}P(r_t,r_{t-1},\boldsymbol{x}_{1:t})  \\ & = 
\sum_{r_{t - 1}=0}^{t-1}  P(r_t,x_t|r_{t-1},\boldsymbol{x}_{1:t-1})P(r_{t-1},\boldsymbol{x}_{1:t-1}) \\ & = 
\sum_{r_{t - 1}=0}^{t-1} \underbrace{P(r_t|r_{t-1})}_{\text{prior}} \underbrace{P(x_t |r_{t},\boldsymbol{x}_{t-1}^{(r_{t-1})})}_{\text{likelihood}}\underbrace{P(r_{t-1},\boldsymbol{x}_{1:t-1})}_{\theta_{t-1}}.
\end{split}
\label{eq:message-passing}
\end{equation}
Equation \eqref{eq:message-passing} defines a forward message passing scheme from $\theta_{t-1}$ to $\theta_{t}$. The likelihood depends only on $\boldsymbol{x}_t^{(r)}$, and the prior is specified in terms of a hazard function $H$:
\[
P(r_t|r_{t-1}) = \begin{cases}
H(r_{t-1}+1) & \mbox{if } r_t = 0 \\
1-H(r_{t-1}+1) & \mbox{if } r_t = r_{t-1} + 1 \\
0 & \text{otherwise}
\end{cases} .
\]
Here, \[
H(x) = \dfrac{P_{\texttt{gap}}(g=x)}{\sum_{t=x}^{\infty}P_{\texttt{gap}}(g=t)},
\]
and $P_{\texttt{gap}}(g)$ is a prior distribution for the interval between changepoints. As with \cite{adams2007bayesian}, this prior distribution is chosen to be the geometric distribution with time scale $\lambda_\mathrm{gap}$.  Thus, the hazard function is constant at the prior hazard rate $1/\lambda_\texttt{gap}$.  Figure \ref{figure:EXOPath} illustrates the path of calculation under the given hazard, where, in red, all possible routes to get a run length of 1 at time 4 are shown.  The EXO calculation incorporates all of these possible paths, giving the exact probability of reaching that point under every path.  

\begin{figure}
	\centering
	\begin{tikzpicture}[inner sep=3.2, scale=1, transform shape]{x=1cm, y=1cm, node distance=1pt}
	\tikzstyle{main node}=[draw, circle,thin]
	\draw[->,>=stealth,thick] (.7,-.25) -- coordinate (y axis mid)(.7,6.4);
	\draw[->,>=stealth,thick] (.7,-.25) -- coordinate (x axis mid)(7.5,-.25);
	\foreach \x in {1,...,7}
	\draw (\x,-0.20) -- (\x,-.30)
	node[anchor=north] at (\x, -.3) {\small \x};
	\foreach \y in {0,...,6}
	\draw (.65,\y) -- (.75,\y) 
	node[anchor=east] {\small \y}; 
	
	\node[below=0.5cm] at (x axis mid) {Time $(t)$};
	\node[above=0.3cm, rotate = 90] at (.4,3) {Run length $(r_t)$};
	
	\node[main node] (A) at (1,0) {};
	\node[main node] (B) [right=19pt of A] {};
	\node[main node] (C) [right=19pt of B] {};
	\node[main node] (D) [right=19pt of C] {};
	\node[main node] (E) [right=19pt of D] {};
	\node[main node] (F) [right=19pt of E] {};
	\node[main node] (G) [right=19pt of F] {};
	
	\node[main node] (K1) [above=19pt of B] {};
	\node[main node] (K2) [above=19pt of C] {};
	\node[main node, cyan!70!black, thick, fill] (K3) [above=19pt of D] {};
	\node[main node] (K4) [above=19pt of E] {};
	\node[main node] (K5) [above=19pt of F] {};
	\node[main node] (K6) [above=19pt of G] {};
	
	\node[main node] (H1) [above=19pt of K2] {};
	\node[main node] (H2) [above=19pt of K3] {};
	\node[main node] (H3) [above=19pt of K4] {};
	\node[main node] (H4) [above=19pt of K5] {};
	\node[main node] (H5) [above=19pt of K6] {};

	\node[main node] (M1) [above=19pt of H2] {};
	\node[main node] (M2) [above=19pt of H3] {};
	\node[main node] (M3) [above=19pt of H4] {};
	\node[main node] (M4) [above=19pt of H5] {};
	
	\node[main node] (N1) [above=19pt of M2] {};
	\node[main node] (N2) [above=19pt of M3] {};
	\node[main node] (N3) [above=19pt of M4] {};
	
	\node[main node] (P1) [above=19pt of N2] {};
	\node[main node] (P2) [above=19pt of N3] {};
	
	\node[main node](Q1)[above = 19pt of P2]{};
	
	\draw[dashed] (C)--(D)--(E);
	\draw[dashed] (K1) --(H1) -- (M1) -- (N1) --(P1)--(Q1);
	\draw[dashed](B) --(K2) --(H2) -- (M2) -- (N2)--(P2) ;
	\draw[dashed](K3) --(H3) --(M3)--(N3);
	\draw[dashed](D) -- (K4) -- (H4)--(M4);
	
	\tikzset{edge/.style = {->,> = latex'}}
	\draw[edge, red!70!black,line width=1mm] (A) to (B);
	\draw[edge, red!70!black,line width=1mm] (B) to (C);
	\draw[edge, red!70!black,line width=1mm] (C) to (K3);
	\draw[edge, red!70!black,line width=1mm] (A)--(K1);
	\draw[edge,  red!70!black,line width=1mm] (K1)--(C);
	
	\draw[edge, green!40!black, line width=1mm] (E) -- (K3);
	\draw[edge, green!40!black, line width=1mm] (F) -- (E);
	\draw[edge, green!40!black, line width=1mm] (H3) -- (K3);
	\draw[edge, green!40!black, line width=1mm] (F) -- (H3);
	\draw[edge, green!40!black, line width=1mm] (M3) -- (H3);
	\draw[edge, green!40!black, line width=1mm] (G) -- (F);
	\draw[edge, green!40!black, line width=1mm] (G) -- (M3);
	\draw[edge, green!40!black, line width=1mm] (N3) -- (M3);
	\draw[edge, green!40!black, line width=1mm] (K6) -- (F);
	\draw[edge, green!40!black, line width=1mm] (H5) -- (K5) -- (E);
	
	\draw[dashed] (K2) -- (D);
	\draw[dashed] (H1) -- (D);
	\draw[dashed] (H2) -- (E);
	\draw[dashed] (M1) -- (E);
	\draw[dashed] (K4) -- (F);
	\draw[dashed] (M2) -- (F);
	\draw[dashed] (N1) -- (F);
	
	\draw[dashed] (K5) -- (G);
	\draw[dashed] (H4) -- (G);
	\draw[dashed] (N2) -- (G);
	\draw[dashed] (P1) -- (G);
	
	%
	\end{tikzpicture}	
	
	\caption{The run length illustrated for $t = 1,\dots,7$.  The blue shaded node gives $r_4 = 1$, where the red path signifies the forward paths possible to reach that run length under the model assumptions.  Additionally, the green lines illustrate the backwards path for LEXO-3, again, under model assumptions.}
	\label{figure:EXOPath}
	
\end{figure}
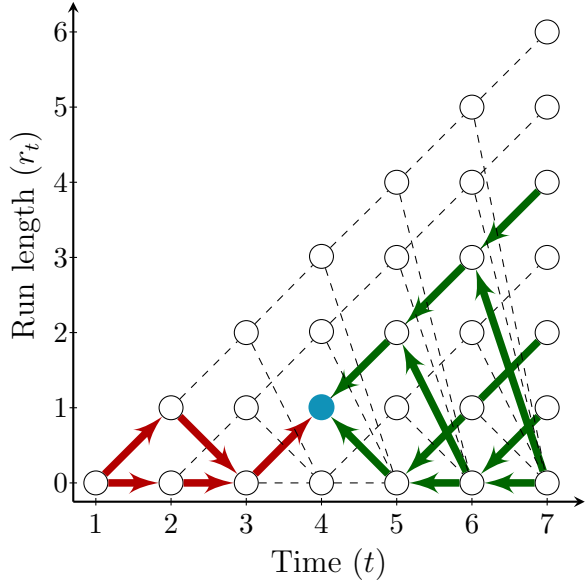

EXO has been incorporated in many places using, for example, Gaussian Processes \cite{knoblauch2018spatio} and Thompson Sampling \cite{mellor2013thompson}, but the forward-message passing scheme of EXO is the most convenient when the data is assumed to follow a distribution that belongs to the exponential family with conjugate priors. In such situation, the likelihood in equation \eqref{eq:message-passing} is characterized by a finite number of sufficient statistics, which in turn can be calculated incrementally as data arrives. More specifically, exponential family likelihoods have the form
\[ P(\bm{x} \vert \bm{\eta}) = h(\bm{x}) \exp(\bm{\eta^\top}\bm{U(x)}-A(\bm{\eta})), 
\]
where \[A(\bm{\eta}) = \log \int h(\bm{x}) \exp(\bm{\eta^\top}\bm{U}(\bm{x})) d\bm{\eta}\] and $\bm{U}(\bm{x})$ denotes the sufficient statistics for $\bm{\eta}$. The conjugate prior for $\bm{\eta}$ has the form
\[ P(\bm{\eta \vert \chi}, \nu) = \tilde{h}(\bm{\eta})\exp\left(\bm{\eta}^\top\bm{\chi}- \nu A(\bm{\eta}) - \tilde{A}(\bm{\chi}, \bm{\nu})\right),
\] where $\bm{\chi}$ and $\nu$ denote the hyperparameter. So, when conditioned on $r_t=r$, the posterior distribution $P(\bm{\eta}_t\vert r_t= r, \bm{x}_{t-1}^{(r_{t-1})})$ has exactly the same form as the prior, but with the new parameter $\bm{\chi}_t^{(r)}$ and $\nu_t^{(r)}$ given by
\[ \bm{\chi}_t^{(r)} = \bm{\chi}_\text{prior} +  \sum_{x_t \in \bm{x}_{t-1}^{(r_{t-1})}} \bm{U}(x_t)
\]
and
\[ \nu_t^{(r)} = \nu_\text{prior} + r_t.
\]
Lastly, the likelihood $P(x_t \vert r_t, \bm{x}_{t-1}^{(r_{t-1})})$ can be calculated by marginalizing the parameter $\bm{\eta}_{\tau(t)}$. Note that, in this updating scheme, if $r_t=0$ then $x_t$ starts a new regime, which implies the likelihood is calculated under the prior value for the hyperparameters. The EXO-algorithm, as developed by \cite{adams2007bayesian}, is summarized in the Algorithm \ref{algo:exo}. 
\begin{algorithm}
	\setstretch{1}
	\caption{EXO Algorithm}
	\label{alg:LEXO}
	\begin{algorithmic}[]
	\Initialize{} 
	\Indent
	\begin{align*}
	&P(r_1 = 0) = 1; ~ P(x_1, r_1) = P(x_1 \vert \bm{\chi}_1^{(0)}, \nu_1^{(0)} ) \\ & \bm{\chi}_1^{(0)} = \bm{\chi}_\text{prior}; ~ \nu_1^{(0)} = \nu_\text{prior}
	\end{align*}
	\EndIndent
	\State Update sufficient statistics:
	\begin{align*}
	\bm{\chi}_2^{(0)} = \bm{\chi}_\text{prior}; \quad   \nu_2^{(0)} = \nu_\text{prior} \\
	\bm{\chi}_2^{(1)} = \bm{\chi}_1^{0} + \bm{U}(x_1); \quad \nu_2^{(1)} = \nu_1^{(0)} +1 \\
	\end{align*}
	\For {$t=2,3,\ldots$} 
		\State Calculate predictive probabilities:
	\[\pi_t^{(r)}=P(x_t \vert r_t = r, \bm{x}_{t-1}^{r_{t-1}}) = P(x_t \vert \bm{\chi}_t^{(r)}, \nu_t^{(r)})
	\]
	\State Calculate growth probabilities:
	\[P(r_t = r_{t-1}+1, \bm{x}_{1:t}) = P(r_{t-1},\bm{x}_{1:{t-1}} )\pi_t^{(r_t)} (1-H) 
	\]
	\State Calculate changepoint probabilities:
	\[P(r_t = 0, \bm{x}_{1:t}) = \sum_{r_{t-1}} P(r_{t-1},\bm{x}_{1:t-1}) \pi_t^{(0)} H
	\]
	\State Calculate marginal probabilities:
	\[P(\bm{x}_{1:t}) = \sum_{r=0}^{t-1} P(r_t=r, \bm{x}_{1:t})
	\]
	\State Normalize:
	\[P(r_t = r \vert \bm{x}_{1:t}) = \frac{P(r_t = r, \bm{x}_{1:t})}{P(\bm{x}_{1:t})}, \quad r = 0, \ldots, t-1
	\]
	\State Update sufficient statistics:
	\begin{align*} &\bm{\chi}_{t+1}^{(0)} = \bm{\chi}_\text{prior} ,~ \nu_{(t+1)}^{(0)} = \nu_\text{prior} \\
	&\bm{\chi}_{t+1}^{(r_t+1)} = \bm{\chi}_t^{(r_t)} + \bm{U}(x_t) ,~ \nu_{t+1}^{(r_t+1)} = \nu_t^{(r_t)} +1 \\	
	\end{align*}
	\EndFor
\end{algorithmic}
\label{algo:exo}	
\end{algorithm}

We briefly mention the standard pruning procedure that is often used with EXO, for example \cite{caron2012line}.  Often, many of the values of the posterior run length are very small.  This is especially true after a changepoint happens, where the run length probabilities shift back toward 0.  The common tactic to handle this is to force any values less than some threshold, say $10^{-5}$, to be 0.  This methodology often leaves to a massive reduction in calculations and memory, and is easily extended to the LEXO framework below.

\subsection{LEXO}
Note that EXO uses data from time 1 to $t$ to make an inference at time $t$; hence the algorithm only makes a forward pass.  We consider the incorporation of a backward pass for estimating the regime parameters at time $t-\ell$, i.e ${r}_{t-\ell}$ and $\boldsymbol{\eta}_{t-\ell}$ for some $\ell \in \mathbb{N}$. We will refer to the methodology as lagged exact online inference with $\ell$ lags (LEXO-$\ell$). In other words, at each time point in the sequence, we refine the inference made in the past. Our algorithm remains online in the sense that once a new datum is observed the past data need not be reused to update the model.  Moreover, the calculations consist of simple operations of already computed values from the EXO calculation, which can be used to give more refined estimates of ${r}_{t-\ell}$, and hence $\eta_{t-\ell}$ over time.  Figure \ref{figure:EXOPath} illustrates the incorporation of three lags, shown in green, into the previous example of getting a run length of $1$ at time $t=4$.  With the incorporation of LEXO-3 an additional 7 possible paths from time 5, 6, and 7 are incorporated, in addition to the 2 possible forward paths already computed via EXO. An example backward path to $r_4 = 1$ is $(r_7, r_6, r_5, r_4) = (0,3,2,1)$.

\begin{theorem}
	Given a sequence of data $\{x_t\}_{t=1}^{\infty}$ follows a product partition model, LEXO-$\ell$ gives exact posterior distributions for the run length for any $\ell \in \mathbb{N}$. Moreover, these run length distributions are updated recursively, with the computation complexity $\mathcal{O}(\ell t)$.
\end{theorem}
\begin{proof}
	We will start with $\ell=1$, so the goal is to compute the distribution $P(r_{t-1} | \bm{x}_{1:{t}})$. By the chain rule, we have
	\begin{equation}
	P(r_{t-1} \vert \bm{x}_{1:t}) = \sum_{r_{t}} P(r_{t-1} \vert r_{t}, \bm{x}_{1:t}) P(r_{t} \vert \bm{x}_{1:t})
	\label{eq:lexo-1}
	\end{equation}
	where $r_t \in \{0, r_{t-1}+1\}$. The second term in the sum of \eqref{eq:lexo-1} is computed by EXO at time $t$.  For the first term, if $r_{t} \neq 0$, then $P(r_{t-1} = r_{t}-1 \vert r_{t}, \bm{x}_{1:t}) = 1$. If $r_{t}=0$, then the chain rule gives
	\begin{align*}
	P(r_{t-1} | r_{t}=0, \bm{x}_{1:t}) &\propto P(r_{t-1}, r_{t}=0, \bm{x}_{1:{t}}) \\& \propto 
	P(x_{t} | r_{t-1}, r_{t}=0, \bm{x}_{1:t-1})  P(r_{t}=0|r_{t-1}, \bm{x}_{1:t-1}) \times   P(r_{t-1}, \bm{x}_{1:t-1}) \\ & \propto 
	P(x_{t} | r_{t}=0, \bm{x}_{1:t-1})  P(r_{t-1}, \bm{x}_{1:t-1}) \\ & \propto P(r_{t-1}  \vert \bm{x}_{1:t-1}),
	\end{align*}
	which is exactly the run length distribution computed by EXO at time $t-1$. In the above derivation
	the third proportionality follows because $P(r_t=0 \vert r_{t-1}, \bm{x}_{1:{t-1}}) = H$, which is a constant. The last proportionality follows because $P(x_{t} | r_{t}=0, \bm{x}_{1:t-1})$
	does not depend on $\bm{x}_{1:t-1}$.  Intuitively, if we have a new regime at $t$, all the past information before that becomes irrelevant. Putting everything together, we have
	\begin{align*}
	P(r_{t-1} = r  | \bm{x}_{1:{t}}) =   P(r_{t} = r+1 | \bm{x}_{1:{t}})  + P(r_{t-1} = r | \bm{x}_{1:t-1}) P(r_{t}=0 | \bm{x}_{1:{t}})
	\end{align*}
	for $r=0,\ldots, t-1$. The above proof shows that the run length distribution at time $t-1$ from LEXO-1 is completely determined by the run length distribution at time $t-1$ and $t$ from EXO. 
	
	Generally, for any $\ell\geq 1$ the distribution of run length from LEXO$-\ell$ at time $t-\ell$ is completely determined by the run length distribution at time $t-\ell$ from EXO and the run length distribution at time $t-(\ell-1)$ from LEXO-$(\ell-1)$. Indeed, 
	\begin{align}
	P(r_{t-\ell} | \bm{x}_{1:t})  = \sum_{r_{t-\ell+1}} P(r_{t-\ell} | r_{t-\ell+1}, \bm{x}_{1:t}) P(r_{t-\ell+1} |\bm{x}_{1:t}).
	\label{eq: LEXO}
	\end{align}	
	where $r_{t-\ell+1} \in \{0, r_{t-\ell}+1\}$. The second term in \eqref{eq: LEXO} is obtained from computation of LEXO-$(\ell-1)$ at time $t-\ell+1$. For the first term, if $r_{t-\ell+1} \neq 0$, then $P(r_{t-\ell} = r_{t-\ell+1}-1 \vert r_{t-\ell+1}, \bm{x}_{1:t}) = 1$. If $r_{t-\ell+1}=0$, then we start a new regime at $t-\ell+1$, so $P(\bm{x}_{(t-\ell+1):t} | r_{t-\ell+1}=0, \bm{x}_{1:t})$
	does not depend on $r_{t-\ell}$. Then we have 
	\begin{align*}
	P(r_{t-\ell} | r_{t-\ell+1}=0, \bm{x}_{1:{t}}) \propto P(r_{t-\ell}, \bm{x}_{1:t-\ell})  \propto P(r_{t-\ell} | \bm{x}_{1:t-\ell}),
	\end{align*} which is exactly the EXO algorithm at time $t-\ell$. Putting everything together, we then have
	\begin{align*}
	P(r_{t-\ell} = r  | \bm{x}_{1:{t}}) = P(r_{t-\ell+1} = r+1 | \bm{x}_{1:{t}}) + P(r_{t-\ell} = r | \bm{x}_{1:t-\ell}) P(r_{t-\ell+1}=0 | \bm{x}_{1:{t}}), 
	\end{align*}
	for $r=0,\ldots, t-\ell-1$. The above proof also shows that, for any $\ell\geq 1$ the distribution of run length from LEXO-$\ell$ algorithm can be computed recursively, starting from EXO. Note that going from lag $\ell-1$ to lag $\ell$ requires exactly one computation of complexity $\mathcal{O}(t)$, so the entire procedure is $\mathcal{O}(\ell t)$.
\end{proof}
\section{Parameter Estimation}

Another natural question while detecting changepoints is to estimate the parameter $\bm{\eta}_t$ at each time $t$. In this section, we consider how to find the posterior distribution of these parameters from LEXO-$\ell$.  First we start with EXO $(\ell=0)$,  
\[
P(\bm{\eta}_t \vert \bm{x}_{1:t}) = \sum_{r_t} P(\bm{\eta}_t | r_t, \bm{x}_{1:t}) P(r_t | \bm{x}_{1:t}),
\]
where the second term $P(r_t | \bm{x}_{1:t})$ is computed from EXO algorithm. For the first term, given $r_t=r$, this is the posterior distribution of $\bm{\eta}_t$ computed from $\bm{x}_t^{(r)}$, i.e, as if only $r$ independent observations from the most recent regime are observed.

For LEXO-$\ell$, the parameter estimate process is more involved. The theorem below specifies that the posterior distribution of $\bm{\eta}_t$ can be computed recursively. 
\begin{theorem}
	Given a sequence of data $\{x_t\}_{t=1}^{\infty}$ follows a product partition model, the posterior distribution of $\bm{\eta}_t$ from LEXO-$\ell$ is a mixture of mixture models and can be computed exactly and recursively for all $t$ and $\ell$.
\end{theorem}
\begin{proof}
	For any lag $\ell \geq 1$, we can write
	\begin{equation}
	P(\bm{\eta}_{t-\ell} | \bm{x}_{1:t}) = \sum_{r_{t-\ell}} P(\bm{\eta}_{t-\ell} | r_{t-\ell}, \bm{x}_{1:t}) P(r_{t-\ell} \vert \bm{x}_{1:t}),
	\label{eq:lexo_l}
	\end{equation}
	so the posterior distribution of $\bm{\eta}_{t-\ell}$ is a mixture of components, $P(\bm{\eta}_{t-\ell} | r_{t-\ell}, \bm{x}_{1:t})$, which are weighted by the posterior distribution of run length from LEXO-$\ell$. We show below that, each component $P(\bm{\eta}_{t-\ell} | r_{t-\ell}, \bm{x}_{1:t})$ is also a mixture of two other distributions. Indeed, we have
	\begin{equation}
	P(\bm{\eta}_{t-\ell} | r_{t-\ell}, \bm{x}_{1:t}) = \sum_{r_{t-\ell+1}} P(\bm{\eta}_{t-\ell} | r_{t-\ell+1}, r_{t-\ell}, \bm{x}_{1:t}) P(r_{t-\ell+1} | r_{t-\ell}, \bm{x}_{1:t}),
	\label{eq: first_term}
	\end{equation}
	where $r_{t-\ell+1} \in \{0 , r_{t-\ell} + 1\}$. For the first term in the sum of \eqref{eq: first_term}, if $r_{t-\ell+1} = 0$ then a new regime begins at $t-\ell+1$, which implies $\bm{\eta}_{t-\ell}$ is independent of $\bm{x}_{(t-\ell+1):(t)}$. Denote $ P(\bm{\eta}_{t-\ell} | r_{t-\ell}, \bm{x}_{1:t}) = \bm{\eta}_{t-\ell}^{(\ell)}$, where the superscript $(\ell)$ in $\bm\eta_{t-\ell}^{(\ell)}$ indicates that $\bm\eta_{t-\ell}$ is estimated with $\ell$ more observations after time $t-\ell$. We then have 
	\begin{align*}
	P(\bm{\eta}_{t-\ell} | r_{t-\ell+1}=0, r_{t-\ell}, \bm{x}_{1:t}) = P(\bm{\eta}_{t-\ell} | r_{t-\ell}, \bm{x}_{1:t-\ell}) = \bm{\eta}_{t-\ell}^{(0)},
	\end{align*}
	which is the parameter estimate from EXO at time $t-\ell$. If $r_{t-\ell+1} = r_{t-\ell} +1 > 0$ the regime continues and $x_{t-\ell}$ and $x_{t-\ell+1}$ is from the same regime; this implies 
	\begin{align*}
	P(\bm{\eta}_{t-\ell} \vert r_{t-\ell+1}=r_{t-\ell}+1, r_{t-\ell}, \bm{x}_{1:t}) = P(\bm{\eta}_{t-\ell+1} \vert r_{t-\ell+1}, \bm{x}_{1:{t}}) = \bm{\eta}_{t-\ell+1}^{(\ell-1)}, 
	\end{align*}  
	which is the estimate at time ${t-(\ell-1)}$ with $\ell-1$ more observations. For the second term in the sum of \eqref{eq: first_term}, the chain rule gives
	$P(r_{t-\ell+1} | r_{t-\ell}, \bm{x}_{1:t}) \propto P(r_{t-\ell} | r_{t-\ell+1}, \bm{x}_{1:t}) P(r_{t-\ell+1} \vert \bm{x}_{1:t})
	$. If $r_{t-\ell+1}=0$, then by the previous argument, $P(r_{t-\ell}|r_{t-\ell+1}=0, \bm{x}_{1:t}) = P(r_{t-\ell} | \bm{x}_{t:t-\ell}).$ If $r_{t-\ell+1} \neq 0$, then $P(r_{t-\ell}= r_{t-\ell+1}-1 |r_{t-\ell+1}, \bm{x}_{1:t})=1.$ Therefore, 
	\begin{align*}
	 P(r_{t-\ell+1} = 0 | r_{t-\ell}, \bm{x}_{1:t} )  \propto
	P(r_{t-\ell} | \bm{x}_{1:t-\ell}) P(r_{t-\ell+1}=0 \vert \bm{x}_{1:t}) \triangleq \alpha_1,
	\end{align*}
	and 
	\begin{align*}
	P(r_{t-\ell+1} = r_{t-\ell} + 1 | r_{t-\ell}, \bm{x}_{1:t} ) \propto
	P(r_{t-\ell+1}=r_{t-\ell}+1| \bm{x}_{1:t}) \triangleq \alpha_2.
	\end{align*}
	Letting $\alpha = \alpha_1/(\alpha_1 + \alpha_2)$, then we now have an explicit form for the distribution of $\bm{\eta}_{t}^{(\ell)}$ as
	\[ \bm{\eta}_{t-\ell}^{(\ell)} = \alpha \bm{\eta}_{t-\ell}^{(0)}  + (1-\alpha) \bm{\eta}_{t-\ell+1}^{(\ell-1)}.
	\]
	This shows that  $\bm{\eta}_{t-\ell}^{(\ell)}$ is itself a mixture of two distributions: 1) the posterior distribution for $\bm{\eta}_t$ from EXO at time $t-\ell$ and 2) the posterior distribution for $\bm{\eta}_{t-\ell+1}$ from LEXO-$(\ell-1)$. Therefore, the distribution of $\bm{\eta}_{t-\ell}^{(\ell)}$ and $P(\bm{\eta}_{t-\ell} | \bm{x}_{t:t})$ can be computed recursively. 
\end{proof}
Applying the properties of mixture model, we can compute the posterior $K^{th}$ moment of $P(\bm{\eta}_{t-\ell} | \bm{x}_{1:t})$ in the two following stages:
\begin{itemize}
	\item Stage 1: Compute $\mathbb{E}\left(\left[\bm{\eta}_{t-\ell}^{(\ell)}\right]^K\right)$ recursively as
	\[
	\mathbb{E}\left(\left[\bm{\eta}_{t-\ell}^{(\ell)}\right]^K\right) = \alpha \mathbb{E}\left(\left[\bm{\eta}_{t-\ell}^{(0)}\right]^K\right) + (1-\alpha) \mathbb{E}\left(\left[\bm{\eta}_{t-\ell+1}^{(\ell-1)}\right]^K \right). \]
	\item Stage 2: Compute the posterior moment:
	\[
	\mathbb{E}\left(\left[\bm{\eta}_{t-\ell} \vert \bm{x}_{1:t}\right]^K\right) = \sum_{r_{t-\ell}} \mathbb{E}\left(\left[\bm{\eta}_{t-\ell}^{(\ell)}\right]^K\right) P(r_{t-\ell}|\bm{x}_{1:t}).
	\]
\end{itemize}

\section{Simulation}

To demonstrate the effect of incorporating lags into the online framework of EXO, we investigate LEXO for changepoint detection under three different settings -- a Normal distributed model with mean shift, a Normal distributed model with precision (inverse variance) shift, and a Poisson distributed model (both mean and variance shift).  Formally these settings can be described in the same order as the following:

\begin{enumerate}
	\item $X_{\tau(t)} \sim N(\mu_{\tau(t)} , 1)$ where $\mu_{1} = 0$ and changes such that $\mu_{\tau(t)} = \mu_{\tau(t) - 1} + 1$.
	\item $X_{\tau(t)} \sim N(0 , 1/\xi_{\tau(t)}^2)$ where $\xi_{1} = 16$ and changes such that $\xi_{\tau(t)} = (1/4)\xi_{\tau(t) - 1}$.
	\item $X_{\tau(t)} \sim \text{Poisson}(\lambda_{\tau(t)})$ where $\lambda_{1} = 1$ and changes such that $\lambda_{\tau(t)} = 4 + \lambda_{\tau(t) -1}$.
\end{enumerate}

The Normal mean and precision shift represent common applications of an online streaming scenario.  The Poisson setting illustrates a common counting problem that can be difficult to detect due to the increasing variability of the model. An example of each of the three processes can be seen in Figure \ref{figure:SimulationExampleProcess}.  For all settings, we impose 5 equally-spaced changepoints; for each setting, 1000 samples were generated. 

\begin{figure}[H]
	\centering
	\makebox[\textwidth][c]{
		\begin{tabular}{c c c}
			
			\begin{subfigure}{.33\textwidth}
				\includegraphics[page = 1 , height = 5cm , trim = {0 1.1cm 0 0cm}, clip]{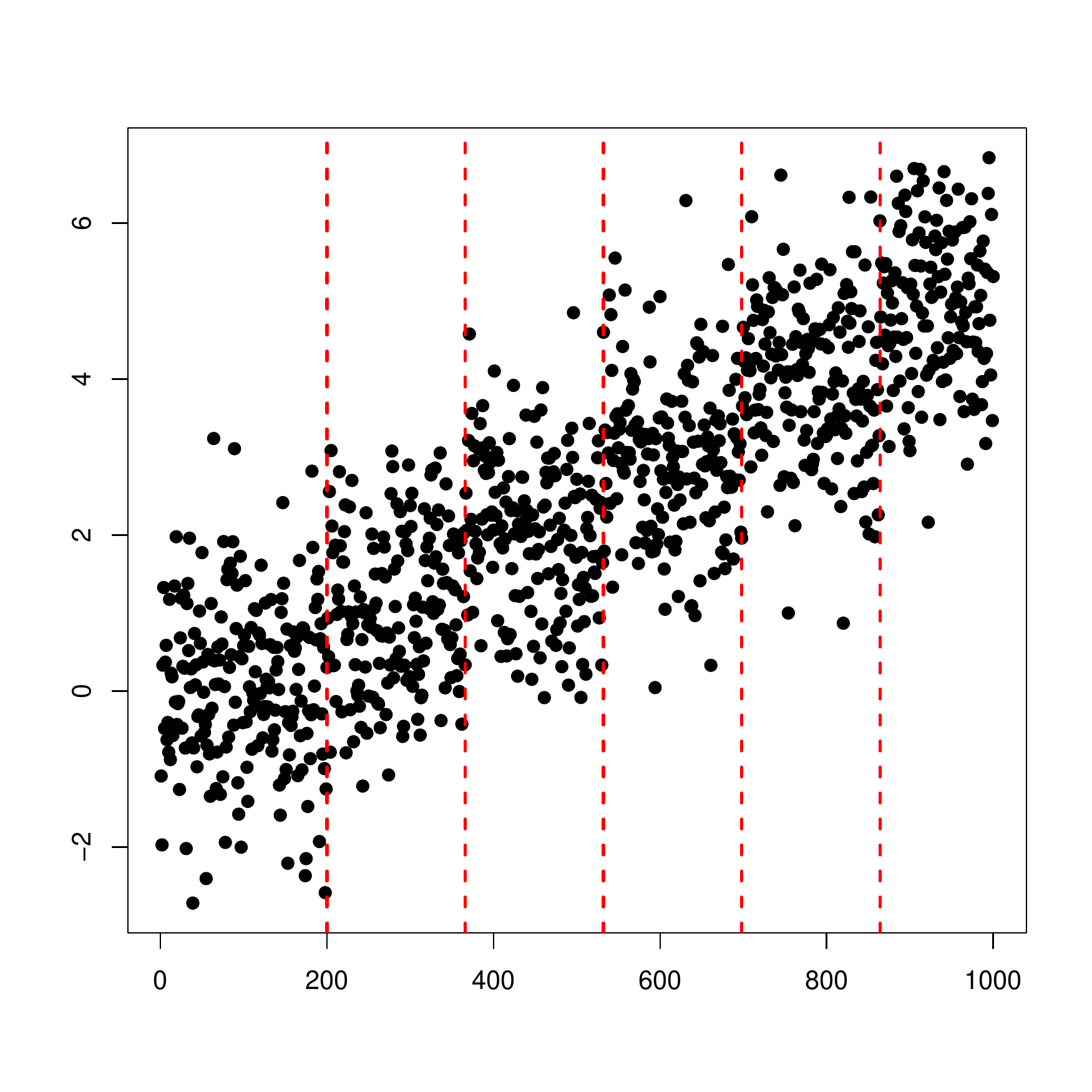}
				\caption{Mean Shift}
			\end{subfigure} &
			
			\begin{subfigure}{.33\textwidth}
				\includegraphics[page = 2 , height = 5cm , trim = {0 1.1cm 0 0cm}, clip]{SampleProcesses.pdf}
				\caption{Precision Shift}
			\end{subfigure} & 
			
			\begin{subfigure}{.33\textwidth}
				\includegraphics[page = 3 , height = 5cm , trim = {0 1.1cm 0 0cm}, clip]{SampleProcesses.pdf}
				\caption{Poisson}
			\end{subfigure}
			
		\end{tabular}
	}
	
	\caption{Random draw of each of the three simulated processes considered in the simulation studies.}
	\label{figure:SimulationExampleProcess}
	
\end{figure}

For each sample, we investigate the performance of the run length estimates and the parameter estimates made by EXO and LEXO-$\ell$, for $\ell = 1, \dots, 30$. The hyperparameters for the prior were chosen to be non-informative, and the hazard rate was chosen to be $H=1/50$.  The difference in estimated run length is inspected from the 1000 samples' median of maximum-a-posteriori (MAP) of the posterior distribution of the run length, which is plotted at each time point. This plot illustrates a common choice of determining a changepoint made by practitioners, see \cite{adams2007bayesian}.  Additionally, the parameter estimates are compared at each time point across simulations by calculating the ratio of posterior MSE of EXO and LEXO. At time $t-\ell$, given the true parameter $\bm\eta_t$, the posterior MSE of LEXO$-\ell$  is defined as
\[ \text{MSE}_t^{(\ell)} = \left(\mathbb{E}\left(\bm\eta_{t-\ell}|\bm{x}_{1:t}\right)-\bm\eta_t\right)^2 + \mathrm{Var}\left(\bm\eta_{t-\ell}|\bm{x}_{1:t}\right)
\]
Note that, LEXO is performing better than EXO in terms of parameter estimation when the ratio is greater than 1.  The average posterior mean across 1000 samples of each parameter is also plotted at each time point.

Regarding the run length distribution of EXO and LEXO, the run length MAP plots can be seen in Figure \ref{figure:SimulationRunLength}.  From each of the processes, a clear shift can be observed from LEXO, correcting overstepping from EXO's run length MAP.  Additionally, the shift converges toward the true changepoint in the process when allowed enough lags to refine the uncertainty in the estimate.  It seems evident that LEXO is refining EXO appropriately, not over correcting past where a process changed.

\begin{figure}[H]
	
	\centering
	\makebox[\textwidth][c]{
		\begin{tabular}{c}
			
			\begin{subfigure}{\textwidth}
				\includegraphics[page = 1 , trim = {.5cm 12.8cm .5cm 1cm}, clip , width = 17cm]{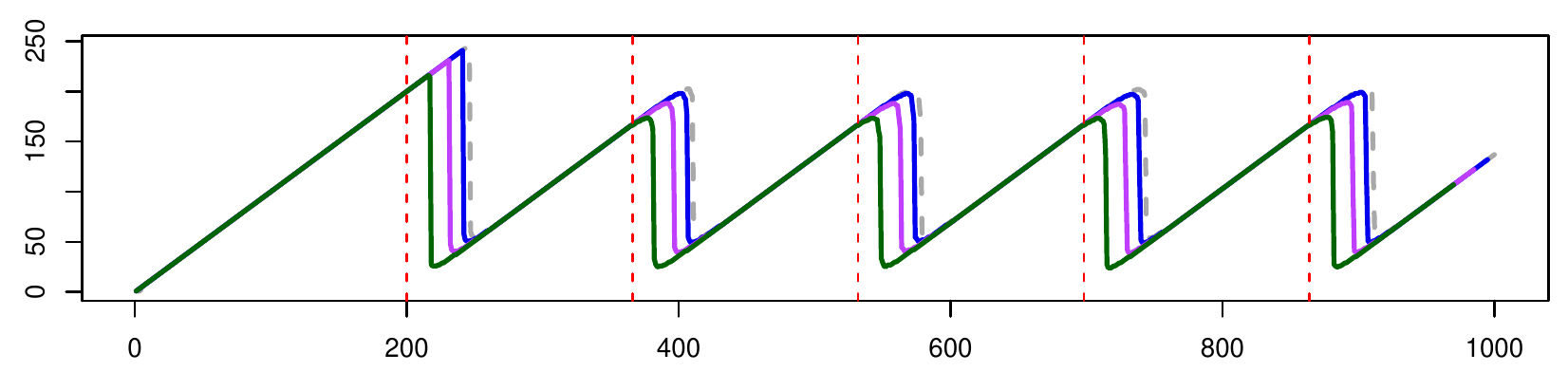}
				\caption{Mean Shift}
			\end{subfigure} \\
			
			\begin{subfigure}{\textwidth}
				\includegraphics[page = 1 , trim = {.5cm 12.8cm .5cm 1cm}, clip , width = 17cm]{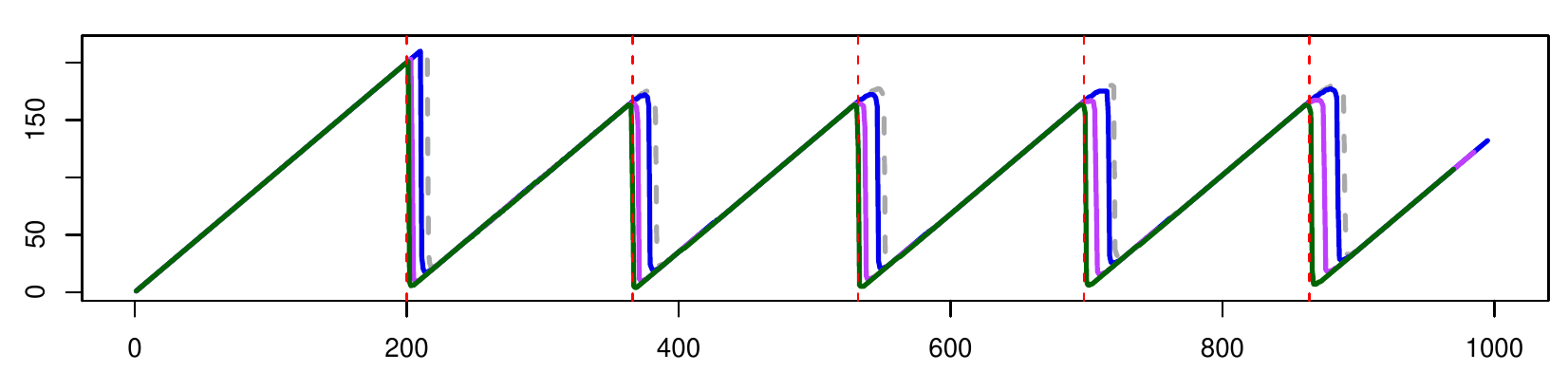}
				\caption{Precision Shift}
			\end{subfigure} \\
			
			\begin{subfigure}{\textwidth}
				\includegraphics[page = 1 , trim = {.5cm 12.8cm .5cm 1cm}, clip , width = 17cm]{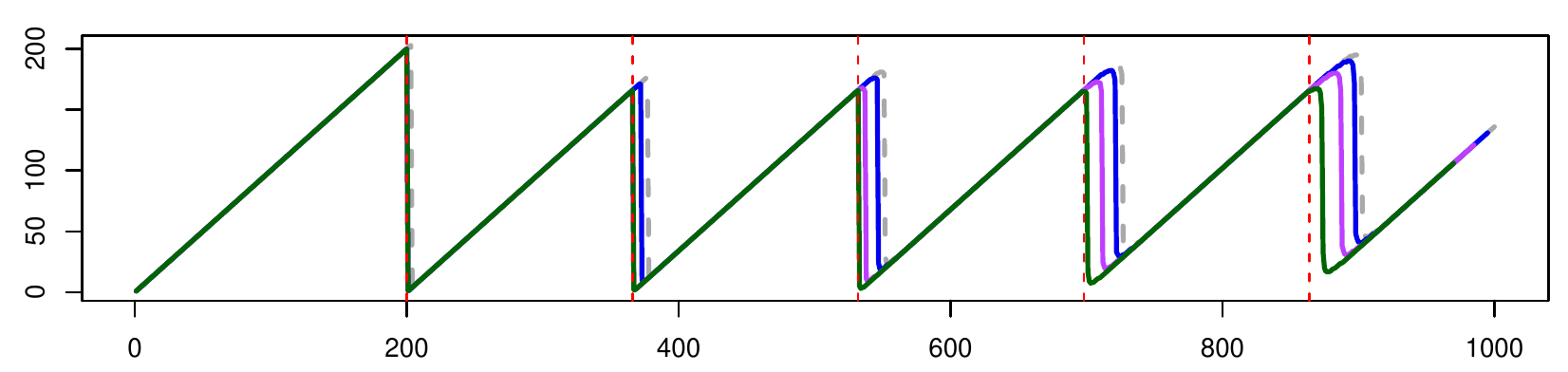}
				\caption{Poisson}
			\end{subfigure}
			
		\end{tabular}
	}
	
	\caption{Median run length, as indicated by the MAP, from 1000 simulations of each process, from EXO (dashed gray line) and LEXO$-\ell$, for $\ell=5$ (solid blue line), $\ell=15$ (solid  purple line), and $\ell=30$ (solid green line).}
	\label{figure:SimulationRunLength}
	
\end{figure}

Inspecting the parameter estimates also shows clear performance gains from LEXO.  Beginning with the MSE ratio plots, observed in the left column of Figure \ref{fig: MSEandMean}, a substantial improvement of the estimates is apparent.  Each of the three processes show an improvement in terms of the MSE ratio, where it is not uncommon to note more than five times better performance. Noticeably, for the precision shift model, the MSE ratio can be up to $1.5\times 10^7$ for LEXO-30. This can be attributed to more data being incorporated at each point from future points, as well as more concentration around the correct run length in the estimated run length distribution. 

Next, the average mean plots are presented in the right column of Figure \ref{fig: MSEandMean}.  The improvement of the parameter estimate is self evident from the plots, where the estimate gets closer and closer to the truth.  EXO's estimate of the precision shift model is so poor that it does not completely appear on the plot with reasonable dimensions.  However, after just incorporating 5 lags, LEXO begins to closely follow the true changes in the model.  The Poisson setting follows in a similar fashion to the mean shift model.

Finally, Table \ref{table:MSE} illustrates the average MSE ratio of EXO over some LEXO-$\ell$ at three time points, $t = $ 197, 200, and 220, corresponding to right before, at, and a fair bit after a changepoint. Overall most of the MSE ratios are greater than $1$; especially for the normal precision shift model, the ratio is much bigger than $1$. This shows considerable benefit of LEXO over EXO. However, right before a changepoint ($t=197$), adding a large number of lags can have a detrimental effect on the performance of the estimates, when the MSE ratio falls below $1$. This occurs because a large number of lags include too many points from the next regime.  However, the detection of the regime changes does not deteriorate with higher lags.  A reasonable solution is to use as higher order a lag as possible for the detection of changepoints, while using a smaller lag for moment calculations.



\begin{figure}[H]

	\centering
	\makebox[\textwidth][c]{
		\begin{tabular}{c}

			\begin{subfigure}{\textwidth}
				\begin{tabular}{c c}
					\includegraphics[page = 1 , height = 6.3cm, trim = {0 1.5cm 0 1.5cm}, clip]{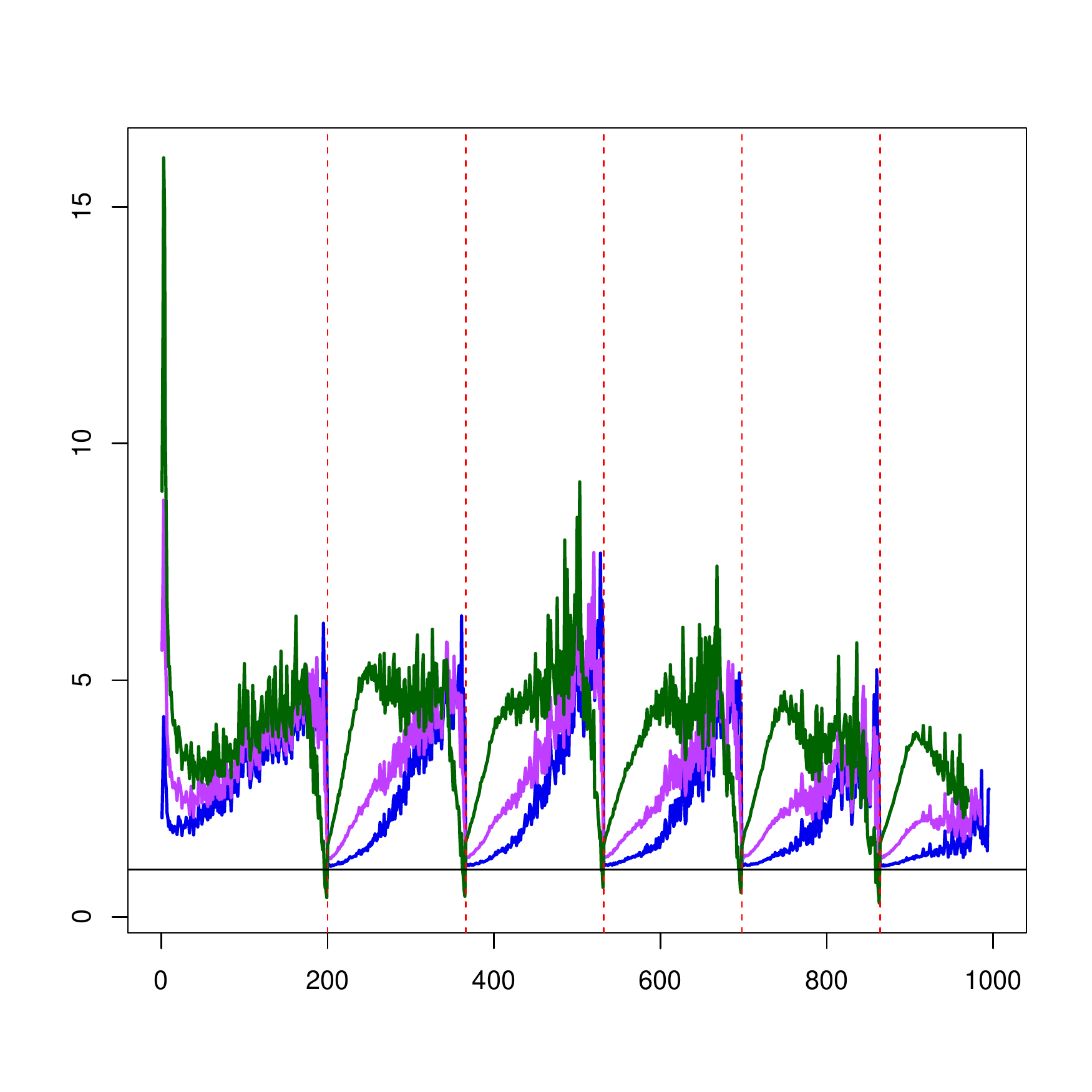} & 
					\includegraphics[page = 1 , height = 6.3cm, trim = {0 1.5cm 0 1.5cm}, clip]{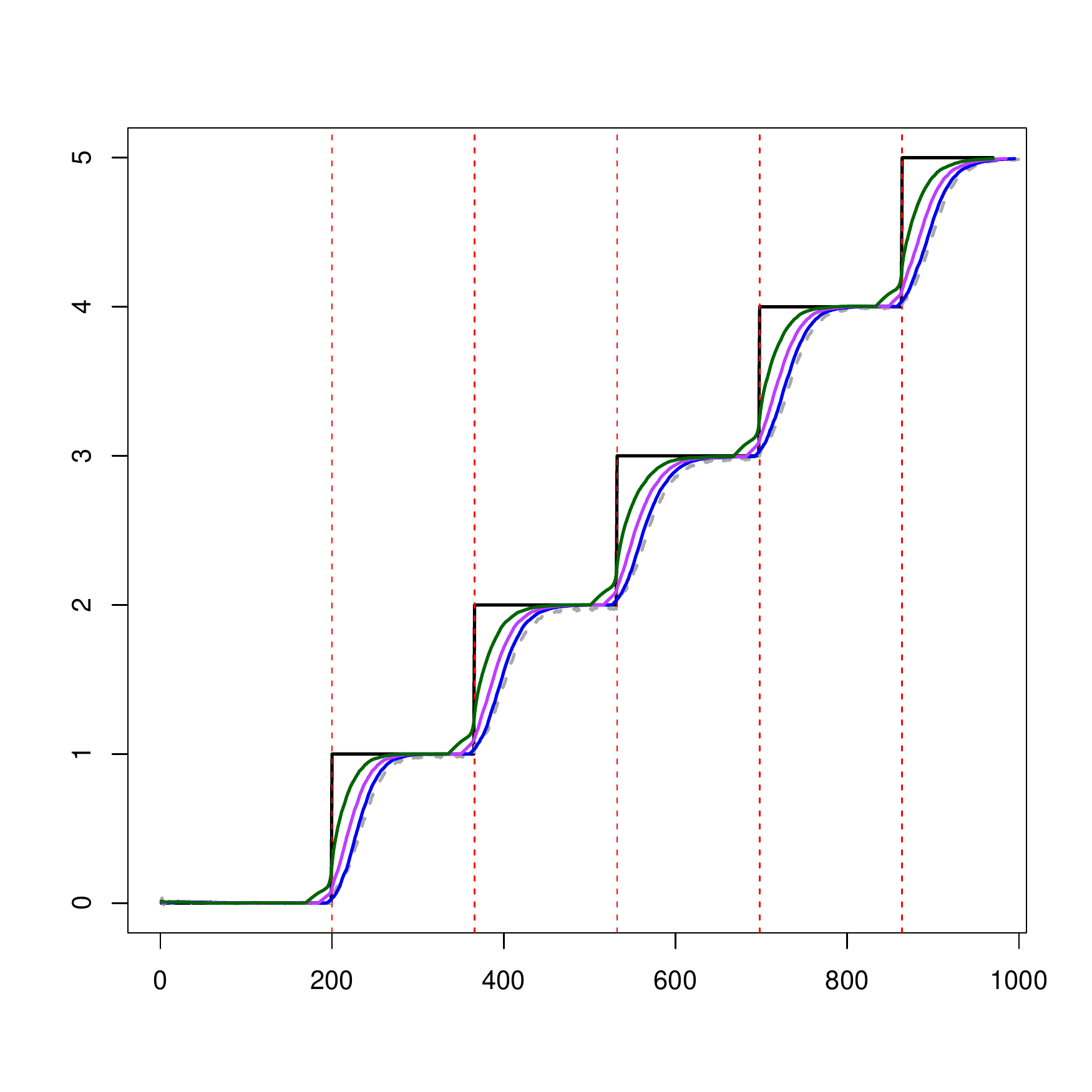}
				\end{tabular}
				\caption{Mean Shift}
			\end{subfigure} \\

			\begin{subfigure}{\textwidth}
				\begin{tabular}{c c}
					\includegraphics[page = 1 , height = 6.3cm, trim = {0 1.5cm 0 1.5cm}, clip]{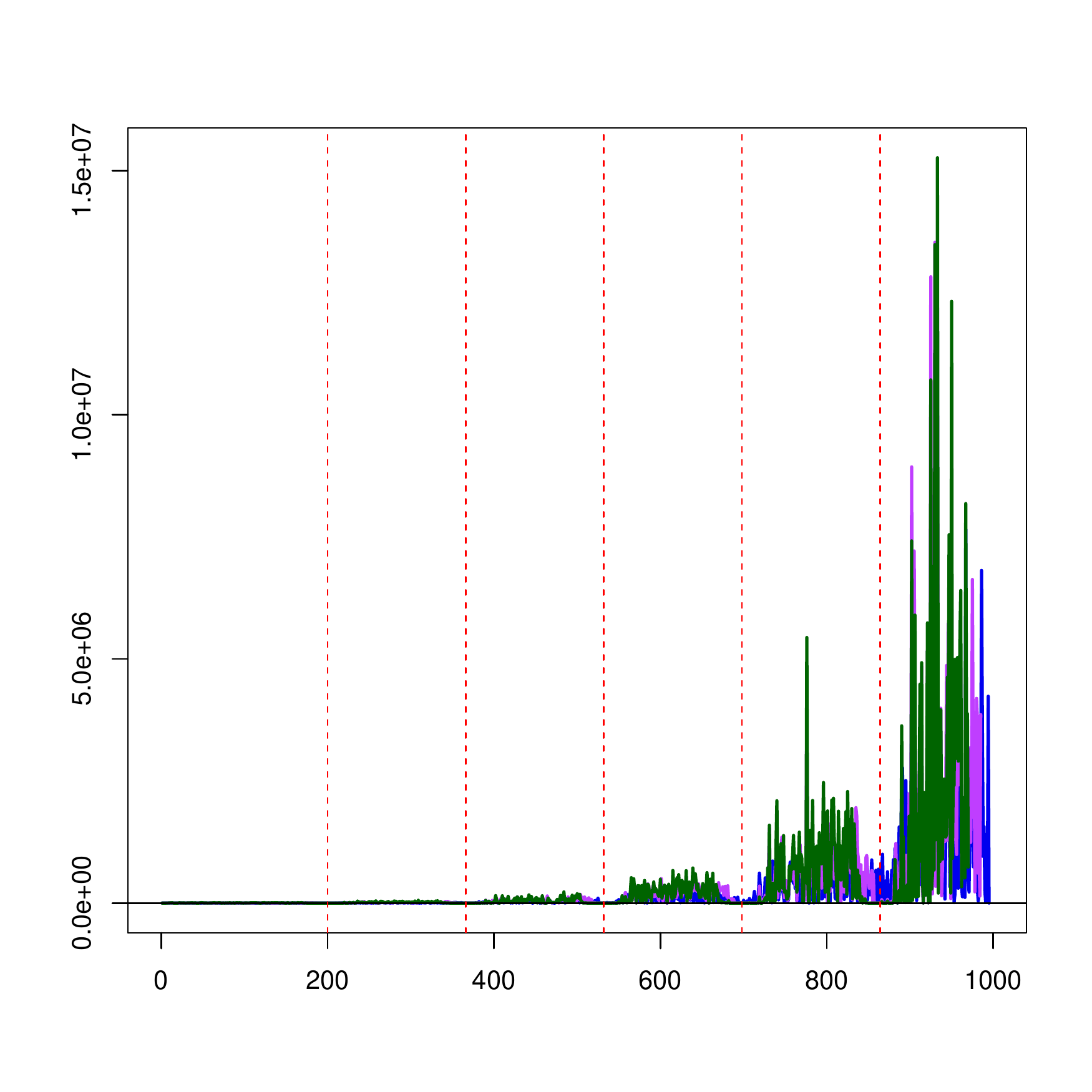} & 
					\includegraphics[page = 1 , height = 6.3cm, trim = {0 1.5cm 0 1.5cm}, clip]{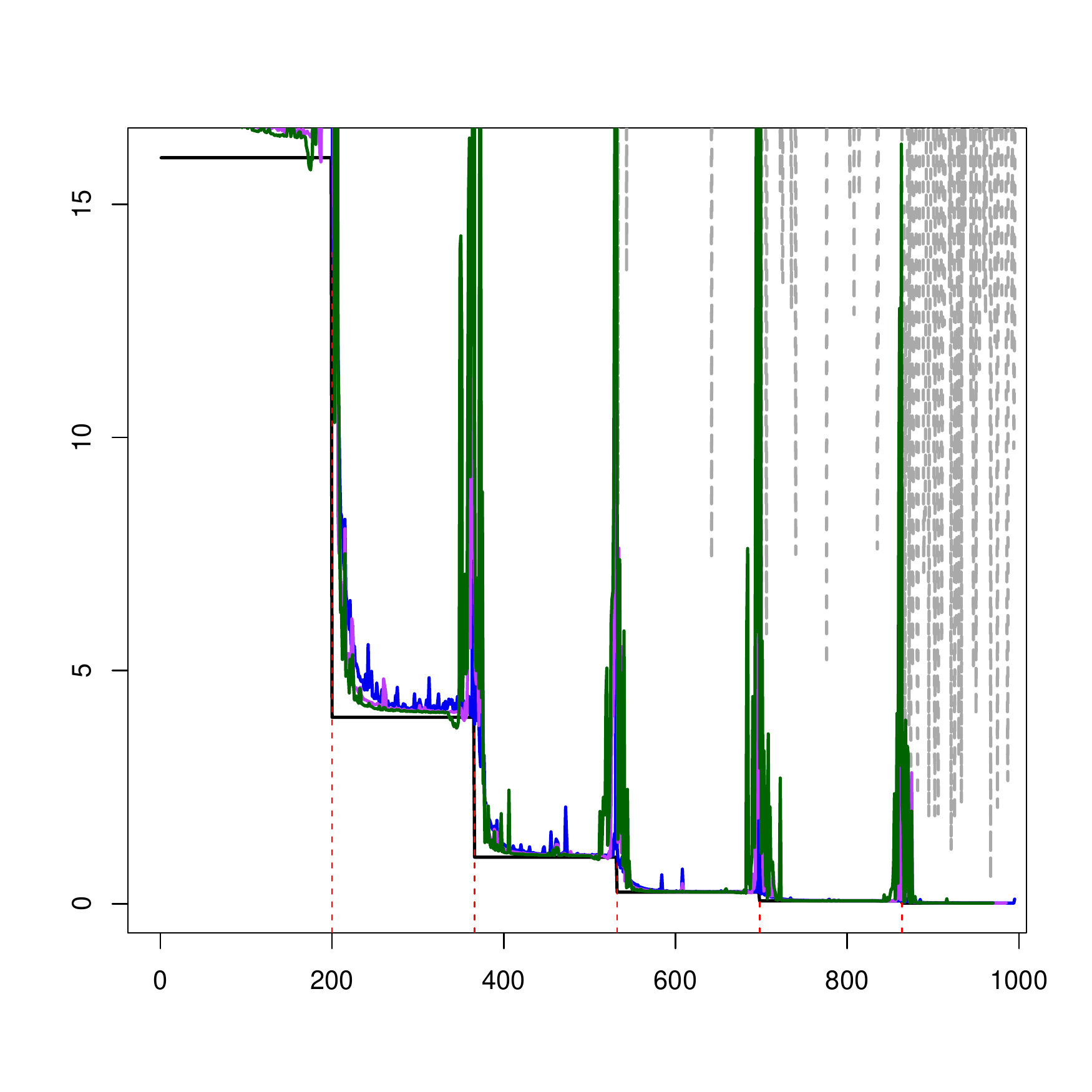}
				\end{tabular}
				\caption{Precision Shift}
			\end{subfigure} \\

			\begin{subfigure}{\textwidth}
				\begin{tabular}{c c}
					\includegraphics[page = 1 , height = 6.3cm, trim = {0 1.5cm 0 1.5cm}, clip]{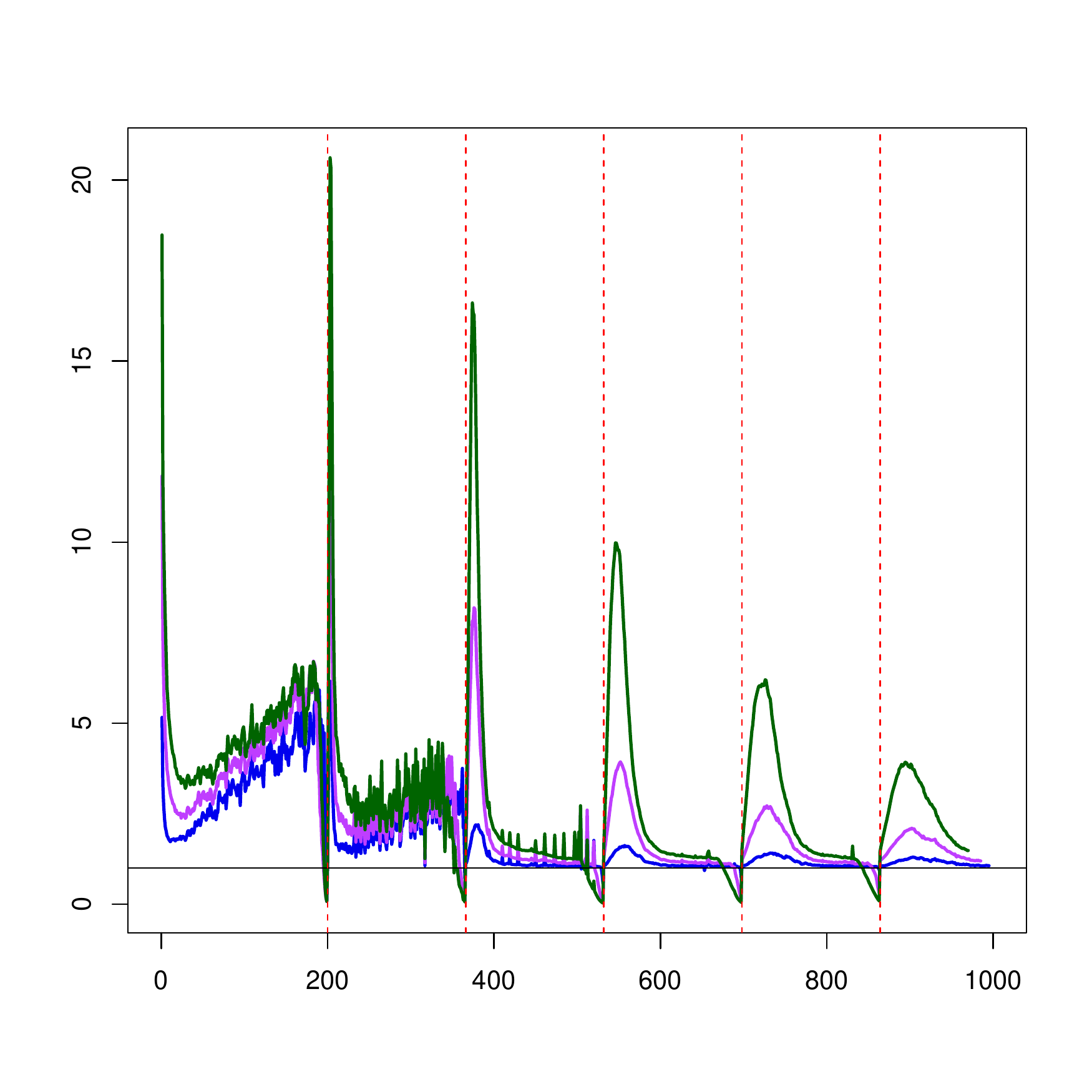} & 
					\includegraphics[page = 1 , height = 6.3cm, trim = {0 1.5cm 0 1.5cm}, clip]{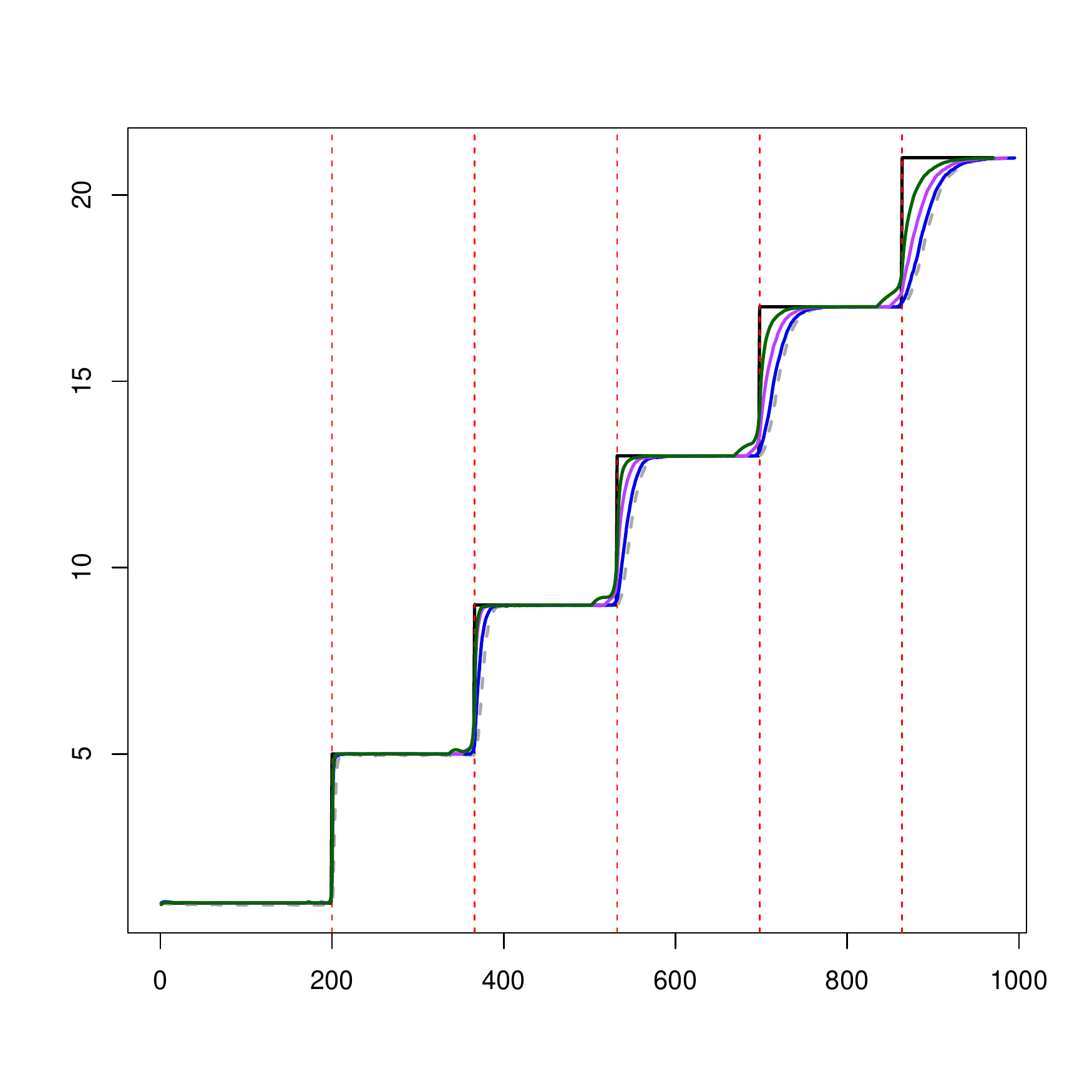}
				\end{tabular}
				\caption{Poisson}
			\end{subfigure} 

		\end{tabular}
	}
	\caption{The left plots represent the ratio of average posterior MSE of EXO vs LEXO-$\ell$ at each time point, $\mathrm{MSE}_t^{(0)}/\mathrm{MSE}_t^{(\ell)}$, while the right plots show average posterior mean across 1000 simulations for EXO (dashed gray line), $\ell=5$ (solid blue line), $\ell=15$ (solid purple line), and $\ell=30$ (solid green line), along with the true value of the parameter (solid black line).}
	\label{fig: MSEandMean}
\end{figure}

\begin{table*}[t]
	\centering
	\caption{MSE ratio of EXO over LEXO-$\ell$ given at $t=197$ (right before a changepoint), $t=200$ (at a changepoint), and $t=220$ (in the middle of a regime)}
	\begin{tabular}{lccccccccc}
		\hline
		Type & Time & \multicolumn{8}{c}{MSE ratio of EXO over LEXO-$\ell$} \\
		\cline{3-10}
		& & 1 & 2 & 3 & 4 & 5 & 10 & 15 & 30 \\ 
		\hline
		\multirow{3}{3cm}{Mean Shift} & 197 & 2.025 & 2.915 & 3.763 & 4.079 & 4.079 & 3.675 & 2.667 & 0.618 \\ 
		& 200 & 1.052 & 1.069 & 1.077 & 1.084 & 1.099 & 1.154 & 1.223 & 1.514 \\ 
		& 220 & 1.050 & 1.072 & 1.100 & 1.125 & 1.155 & 1.349 & 1.606 & 3.148 \\ 
		\hline
		\multirow{3}{3cm}{Precision Shift} & 197& 5.083 & 228.355 & 30.520 & 40.400 & 37.375 & 7.354 & 4.956 & 3.398 \\ 
		& 200& 36.372 & 63.331 & 56.260 & 33.187 & 37.700 & 11.574 & 8.221 & 4.972 \\ 
		& 220 & 37.096 & 179.673 & 3906 & 2024 & 10760 & 14027 & 13469& 14063 \\ 
		\hline
		\multirow{3}{3cm}{Poisson} & 197& 1.341 & 2.046 & 2.463 & 4.528 & 1.615 & 0.363 & 0.384 & 0.469 \\ 
		& 200& 1.045 & 1.074 & 1.123 & 1.173 & 1.212 & 1.245 & 1.239 & 1.232 \\ 
		& 220& 1.371 & 1.423 & 1.494 & 1.542 & 1.613 & 1.979 & 2.326 & 3.530 \\ 
		\hline
	\end{tabular}
	\label{table:MSE}
\end{table*}

\section{Data Analysis}
\subsection{Dow Jones Return Rate}
In the period from mid-1972 to mid-1975 several major political events occurred that had potentially significant impacts on the US macroeconomy, including the Watergate scandal and the OPEC embargo. We examine the daily returns of the Dow Jones Industrial Average from July 3, 1972 to June 30, 1975 to detect abrupt changes in the underlying distribution of the return. The daily return of day $t$ with closing price $p_{t}^{\text{close}}$ is defined as
\[ 
R_t = \frac{p_t^\text{close}}{p_{t-1}^\text{close}} -1,
\]
which was modeled as a Gaussian distribution with mean $0$ with unknown variance. The data on closing price is publicly available through Google Finance.  Of specific interest is the change in the volatility of the return; in other words, we focus on estimating the precision (inverse variance) of the underlying Gaussian distribution. This dataset is also analyzed in the \cite{adams2007bayesian} paper and is plotted in Figure \ref{figure:DowJones} (top).  

We ran EXO and LEXO up to lag $\ell=100$  on the data set, with a gamma prior on the precision with shape and rate of 1 and $10^{-4}$, respectively. The hazard rate was fixed at $\lambda=1/250$. Figure \ref{figure:DowJones} (middle) gives the run length given by the MAP of the run length distribution.  Figure \ref{figure:DowJones} (bottom) illustrates the posterior mean of the precision over time of the return given by EXO and LEXO-$\ell$, where $\ell = 5, 30, 50$, $100$. 

\begin{figure}[H]
	
	\centering
	\makebox[\textwidth][c]{
		\begin{tabular}{c}
			
			\includegraphics[page = 2 , trim = {0 13.5cm 0 1.3cm}, clip , width = 17cm]{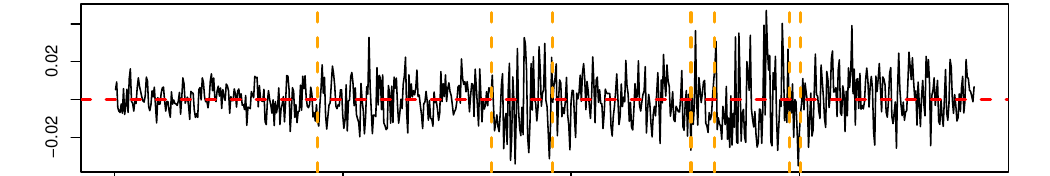} \\
			\includegraphics[page = 4 , trim = {0 13.5cm 0 1.3cm}, clip , width = 17cm]{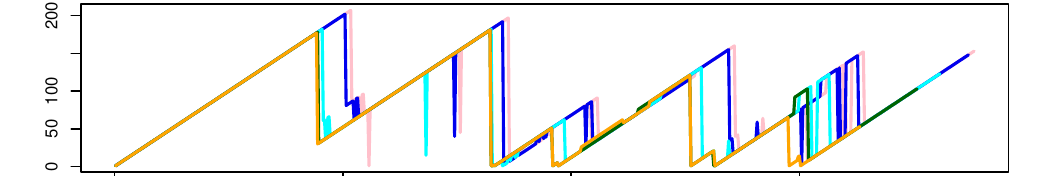} \\
			\includegraphics[page = 5 , trim = {0 12.9cm 0 1.3cm}, clip , width = 17cm]{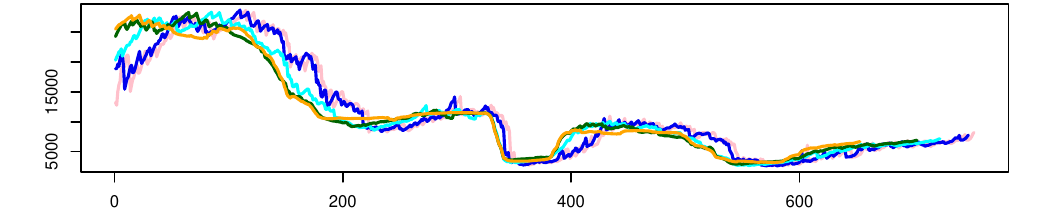} \\
			\includegraphics[page = 6 , trim = {0 15.7cm 0 1.3cm}, clip , width = 17cm]{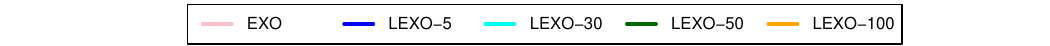} \\
			
		\end{tabular}
	}
	
	\caption{Dow Jones Industrial Average analysis' plots.  Top -- Raw data, with 0 plotted horizontally in red and the locations of the changepoints found by LEXO-100 in orange.  Middle -- Run length for the corresponding lags.  Bottom -- Posterior mean for the precision, for each time, for each lag.}
	\label{figure:DowJones}
	
\end{figure}

Figure \ref{figure:DowJones} shows that the higher lags (50 and 100) detect two additional changepoints near $t = 500$ and at $t=591$.  Additionally, the MAP of run length distribution and the posterior means of the precision given by LEXO-$\ell$ are much more stable than those given by EXO. Table \ref{section5:dj_table} below indicates changepoints and events that are likely associated with them.

\begin{table}[H]
	\centering
	\caption{Date and events that are potientally associated with the detected changepoints for the Dow Jones Return dataset.}
	\begin{tabular}{p{1cm}p{2cm} l l l}
		EXO & LEXO-100 & Event & Description & Date \\ \hline
		208 & 178 & 142 & US involvement with Vietnam War ends  & 1/27/1973 \\
		346 & 330 & 327 & OPEC oil embargo begins & 10/19/1973 \\
		424 & 384 & 379 & Nixon refuses to give tapes to the Senate   & 1/3/1974 \\
		NA & 505 & 504 & Threshold Test Ban Treaty signed  & 7/4/1974 \\
		544 & 526 & 526 & Nixon's involvement in Watergate revealed & 8/5/1974 \\
		NA & 591 & 591 & Democrats take control in Congress & 11/5/1974 \\
		657 & 601 & 602 & AT\&T anti-trust suit filed & 11/20/1974 \\ 
		\hline
	\end{tabular}
\vspace{1ex}

\raggedright  \small Note: The first two columns give the index where a sudden MAP shift occurs for EXO and LEXO-100, respectively.  For EXO, we ignore random jumps that appear to be mistakes from the procedure.  The event is the index where the associated even takes place, where if event occurs on a weekend it is put with the Friday of that week.
	\label{section5:dj_table}

\end{table}

\subsection{Coal Mine Disaster Data}

This classic dataset gives counts of coal mining accidents that killed at least ten or more in 112 years from March 15, 1851 and March 22, 1962. We model the number of accidents of each year in the period $x_t$ following a Poisson distribution with unknown rate. We ran EXO and LEXO up to lag $\ell=30$ on the dataset with a gamma prior using a shape and rate of 1 and $10^{-4}$. Figure \ref{figure:CoalMine} (top) shows the posterior mean of the rate parameter overlayed with the data.  Figure \ref{figure:CoalMine} (bottom) shows the run length associated with the MAP of the run length distribution. 

\begin{figure}[H]
	
	\centering
	\makebox[\textwidth][c]{
		\begin{tabular}{c}
			
			\includegraphics[page = 2 , trim = {0 13.5cm 0 1.3cm}, clip , width = 17cm]{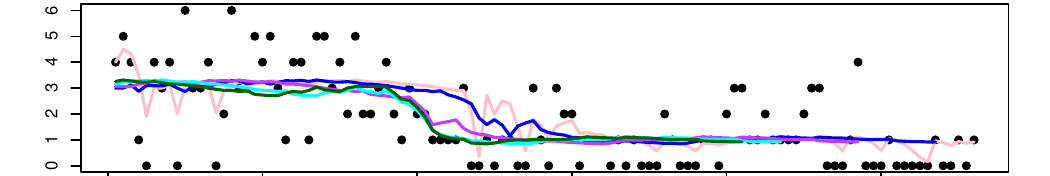} \\
			\includegraphics[page = 3 , trim = {0 12.9cm 0 1.3cm}, clip , width = 17cm]{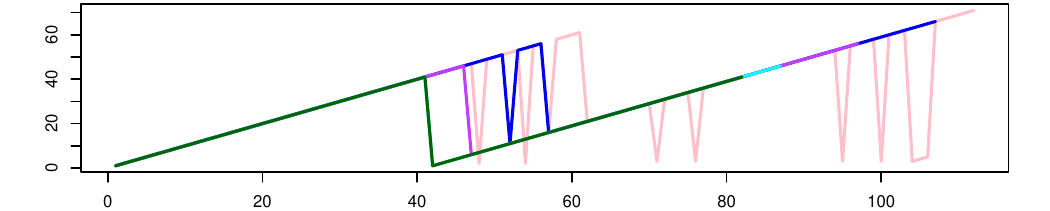} \\
			\includegraphics[page = 5 , trim = {0 15.7cm 0 1.3cm}, clip , width = 17cm]{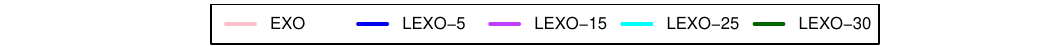} \\
			
		\end{tabular}
	}
	
	\caption{The coal mining analysis plots.  The top illustrates the plotted data with posterior means of the Poisson's rate parameter for corresponding LEXO.  The bottom illustrates the run length found by the MAP of the posterior run length distribution.}
	\label{figure:CoalMine}
	
\end{figure}

Figure \ref{figure:CoalMine} shows that the MAP of the run length distribution and the posterior mean of the rate given by EXO is very unstable. The more lags added, the more stable the plots became. The lines corresponding to LEXO-15, LEXO-25, and LEXO-30 are stable, with no sporadic changes in the MAP of the run length distribution over time. The MAP of the run length seemingly converges to $t = 41$ as LEXO-25 and LEXO-30 both find and settle at this point.  Historically, this has been point identified by other changepoint methodologies,  corresponding to the introduction of the Coal Mines Regulation Act in 1887.  The higher lags illustrate the mean change after the regulation, showing a drop from a rate of about 3 to a rate of around 1.

\section{Conclusion}
The introduction of a backward pass into the online changepoint framework brings much needed refinement in the inference.  The EXO framework is naturally susceptible to outliers and can take several time points before finding a new regime.  We illustrate LEXO-$\ell$ can refine the run length distributions estimated by EXO, including some level of convergence to true changepoints and better parameter estimates.  The LEXO framework is an iterative process, allowing for continually updated regime estimates for the practicioneer to use.  Moreover, for high-frequency data sets where many offline methods that rely on Monte Carlo Markov Chain (MCMC) are not scalable, LEXO could be used with reasonable confidence and computational power.  

\begin{center}
	{\large\bf SUPPLEMENTAL MATERIAL}
\end{center}
	R code that implements the LEXO algorithm for the three changepoint models can be found in the Github at \texttt{github.com/lnghiemum/LEXO}.

\bibliographystyle{biom}      
\bibliography{cites}
\end{document}